\documentclass{article}
\usepackage[margin=1in]{geometry}

\usepackage[utf8]{inputenc}
\usepackage[T1]{fontenc}    
\usepackage{url}            
\usepackage{booktabs}       
\usepackage{amsfonts}       
\usepackage{nicefrac}       
\usepackage{microtype}      
\usepackage{graphicx}
\usepackage[page]{appendix}
\usepackage{mathtools}
\usepackage[export]{adjustbox}

\DeclarePairedDelimiter{\norm}{\lVert}{\rVert}

\usepackage[colorinlistoftodos,textsize=tiny,textwidth=60pt]{todonotes}
\newcommand{\Comments}{1}
\newcommand{\mynote}[2]{\ifnum\Comments=1\textcolor{#1}{#2}\fi}
\newcommand{\mytodo}[2]{\ifnum\Comments=1\todo[linecolor=#1!80!black,backgroundcolor=#1,bordercolor=#1!80!black]{#2}\fi}
\newcommand{\btw}[1]{}

\title{Efficient Online Learning of Optimal Rankings: Dimensionality Reduction via Gradient Descent}
\makeatletter
\def\blfootnote{\xdef\@thefnmark{}\@footnotetext}
\makeatother

\author{%
  Dimitris Fotakis\thanks{National Technical University of Athens, fotakis@cs.ntua.gr}\rule{.085\textwidth}{0pt}  \\
  \and 
  Thanasis Lianeas\thanks{National Technical University of Athens, lianeas@corelab.ntua.gr} \rule{.085\textwidth}{0pt}  \\
  \and 
  Georgios Piliouras\thanks{Singapore University of Technology and Design, georgios@sutd.edu.sg} \rule{.085\textwidth}{0pt} \\
  \and  
  Stratis Skoulakis\thanks{Singapore University of Technology and Design, efstratios@sutd.edu.sg} \rule{.085\textwidth}{0pt} \\
}
\date{}

\usepackage{microtype}
\usepackage{graphicx}
\usepackage{subfigure}
\usepackage{booktabs} 
\usepackage{algpseudocode}

\usepackage[english]{babel}
\usepackage{amsmath}
\usepackage{amssymb,float}
\usepackage{amsthm}
\usepackage{bm}
\usepackage{algorithm,algorithmicx}
\usepackage[colorlinks]{hyperref}

\newtheorem{theorem}{Theorem}

\newtheorem{lemma}{Lemma}
\newtheorem{corollary}{Corollary}
\newtheorem{remark}{Remark}
\newtheorem{definition}{Definition}
\newtheorem{question}{Question}
\newtheorem{observation}{Observation}
\newtheorem{problem}{Problem}
\newtheorem{example}{Example}
\newtheorem{claim}{Claim}

\def\E{\mathop{{}\mathbb{E}}} 

\newcommand{\eps}{\epsilon}

\DeclareMathOperator*{\argmax}{arg\,max}
\DeclareMathOperator*{\argmin}{arg\,min}
\allowdisplaybreaks

\DeclareMathOperator{\acost}{AccessCost}

\newcommand{\gms}{$\mathrm{GMSSC}$}
\newcommand{\ms}{$\mathrm{MSSC}$}
\newcommand{\Reals}{\mathbb{R}}

\begin{document}

\maketitle

\begin{abstract}
\noindent We consider a natural model of online preference aggregation, where sets of preferred items $R_1, R_2, \ldots, R_t$ along with a demand for $k_t$ items in each $R_t$, appear online. Without prior knowledge of $(R_t, k_t)$, the learner maintains a ranking $\pi_t$ aiming that at least $k_t$ items from $R_t$ appear high in $\pi_t$. This is a fundamental problem in preference aggregation with applications to, e.g., ordering product or news items in web pages based on user scrolling and click patterns. The widely studied \emph{Generalized Min-Sum-Set-Cover} (GMSSC) problem serves as a formal model for the setting above. GMSSC is NP-hard and the standard application of no-regret online learning algorithms is computationally inefficient, because they operate in the space of rankings. In this work, we show how to achieve low regret for GMSSC in polynomial-time. We employ dimensionality reduction from rankings to the space of doubly stochastic matrices, where we apply Online Gradient Descent. A key step is to show how subgradients can be computed efficiently, by solving the dual of a configuration LP. Using oblivious deterministic and randomized rounding schemes, we map doubly stochastic matrices back to rankings with a small loss in the GMSSC objective. 
\end{abstract}

\section{Introduction}
\label{sec:intro}
In applications where items are presented to the users sequentially (e.g., web search, news, online shopping, paper bidding), the item ranking is of paramount importance (see e.g., \cite{StreeterGK09,CKMS01,FSR18,DBLP:journals/jmlr/YasutakeHTT12,PT18}). More often than not, only the items at the first few slots are immediately visible and the users may need to scroll down, in an attempt to discover items that fit their interests best. If this does not happen soon enough, the users get disappointed and either leave the service (in case of news or online shopping, see e.g., the empirical evidence presented in \cite{DGMM20}) or settle on a suboptimal action (in case of paper bidding, see e.g., \cite{GP13}). 


To mitigate such situations and increase user retention, modern online services highly optimize item rankings based on user scrolling and click patterns. Each user $t$ is typically represented by her set of preferred items (or item categories) $R_t$\,. The goal is to maintain an item ranking $\pi_t$ online such that each new user $t$ finds enough of her favorite items at relatively high positions in $\pi_t$ (``enough'' is typically user and application dependent). A typical (but somewhat simplifying) assumption is that the user dis-utility is proportional to how deep in $\pi_t$ the user should reach before that happens. 

The widely studied \emph{Generalized Min-Sum Set Cover} (\gms{}) problem (see e.g., \cite{Im16} for a short survey) provides an elegant formal model for the practical setting above. In (the offline version of) \gms{}, we are given a set $U = \{1, \ldots, n\}$ of $n$ items and a sequence of requests $R_1, \ldots, R_T \subseteq U$. Each request $R \subseteq U$ is associated with a demand (or covering requirement) $\mathrm{K}(R) \in \{1, \ldots, |R|\}$. The \emph{access cost} of a request $R$ wrt. an item ranking (or permutation) $\pi$ is the index of the $\mathrm{K}(R)$-th element from $R$ in $\pi$. Formally, 
\begin{equation}\label{eq:access_cost}
\acost(\pi, R) = \{\text{the first index up to which } \mathrm{K}(R) \text{ elements of }R \text{ appear in } \pi\}.  
\end{equation}
The goal is to compute a permutation $\pi^\ast \in [n!]$ of the items in $U$ with minimum total access cost, i.e., $\pi^\ast = \argmin_{\pi \in [n!]}\sum_{t=1}^T \acost(\pi,R_t)$. 

Due to its mathematical elegance and its connections to many practical applications,  \gms{} and its variants have received significant research attention \cite{HA05,AGY09,AG11,ImNZ16}. The special case where the covering requirement is $\mathrm{K}(R_t)=1$ for all requests $R_t$ is known as \emph{Min-Sum Set Cover} (\ms). \ms{} is NP-hard, admits a natural greedy $4$-approximation algorithm and is inapproximable in polynomial time within any ratio less than $4$, unless $\mathrm{P}=\mathrm{NP}$ \cite{FLP04}. Approximation algorithms for  \gms{} have been considered in a sequence of papers \cite{BGK10,SW11,ImSZ14} with the state of the art approximation ratio being $12.5$. Closing the approximability gap, between $4$ and $12.5$, for \gms{} remains an interesting open question. 

\noindent\textbf{Generalized Min-Sum Set Cover and Online Learning.} 
Virtually all previous work on \gms{} (the recent work of \cite{FKKSV20} is the only exception) assumes that the algorithm knows the request sequence and the covering requirements well in advance. However, in the practical item ranking setting considered above, one should maintain a high quality ranking online, based on little (if any) information about the favorite items and the demand of new users. 
%
%

Motivated by that, we study \gms{} as an \textit{online learning} problem \cite{H16}. I.e., we consider a $\textit{learner}$ that selects permutations over time (without knowledge of future requests), trying to minimize her total access cost, and an adversary that selects requests $R_1,\ldots,R_T$ and their covering requirements, trying to maximize the learner's total access cost. Specifically, at each round $t \geq 1$,
\begin{enumerate}
    \item The learner selects a permutation $\pi_t$ over the $n$ items, i.e., $\pi_t \in [n!]$.
    
    \item The adversary selects a request $R_t$ with
    covering requirement $\mathrm{K}(R_t)$.
    
    \item The learner incurs a cost equal to $\acost(\pi_t,R_t)$. 
\end{enumerate}

Based on the past requests $R_1,\ldots,R_{t-1}$ only, an \textit{online learning algorithm} selects (possibly with the use of randomization) a permutation $\pi_t$ trying to achieve a total (expected) access cost as close as possible to the total access cost of the optimal permutation $\pi^\ast$. If the cost of the online learning algorithm is at most $\alpha$ times the cost of the optimal permutation, the algorithm is $\alpha$\textit{-regret} \cite{H16}. If $\alpha = 1$, the algorithm is \textit{no-regret}. In this work, we investigate the following question:


\begin{question}\label{q:main}
Is there an online learning algorithm for \gms{} that 
runs in polynomial time and 
achieves $\alpha$-regret, for some small constant $\alpha \geq 1$?
\end{question}

Despite a huge volume of work on efficient online learning algorithms and the rich literature on approximation algorithms for \gms{}, Question~\ref{q:main} remains challenging and wide open. 
%
Although the \emph{Multiplicative Weights Update} (MWU) algorithm, developed for the general problem of \emph{Learning from Expert Advice}, achieves no-regret for \gms{}, it does not run in polynomial-time. In fact, MWU treats each permutation as a different expert and maintains a weight vector of size $n!$. Even worse, this is inherent to \gms{}, due to the inapproximability result of \cite{FLP04}. Hence, unless $\mathrm{P} = \mathrm{NP}$, MWU's exponential requirements could not be circumvented by a more clever \gms{}-specific implementation, because any polynomial-time $\alpha$-regret online learning algorithm can be turned into a polynomial-time $\alpha$-approximation algorithm for \gms.  
Moreover, the results of \cite{KKL07} on obtaining computationally efficient $\alpha$-regret online learning algorithms from known polynomial time $\alpha$-approximation algorithms for NP-hard optimization problems do not apply to optimizing non-linear objectives (such as the access cost in \gms{}) over permutations.

\noindent\textbf{Our Approach and Techniques.} 
Departing from previous work, which was mostly focused on black-box reductions from polynomial-time algorithms to polynomial-time online learning algorithms, e.g., \cite{KV03,KKL07}, we carefully exploit the structure of permutations and \gms{}, and present polynomial-time low-regret online learning deterministic and randomized algorithms for \gms{}, based on dimensionality reduction and Online Projected Gradient Descent. 

Our approach consists of two major steps. The first step is to provide an efficient no-regret polynomial-time 
learning algorithm for a relaxation of \gms{} defined on doubly stochastic matrices. To optimize over doubly stochastic matrices, the learner needs to maintain only $n^2$ values, instead of the $n!$ values required to directly describe distributions over permutations. This \emph{dimensionality reduction} step allows for a polynomial-time no-regret online algorithm for the relaxed version of \gms. 

The second step is to provide computationally efficient (deterministic and randomized) online rounding schemes that map doubly stochastic matrices back to probability distributions over permutations. The main challenge is to guarantee that the expected access cost of the (possibly random) permutation obtained by rounding is within a factor of $\alpha$ from the access cost of the doubly stochastic matrix representing the solution to the relaxed problem. Once such a bound is established, it directly translates to an $\alpha$-regret online learning algorithm with respect to the optimal permutation for \gms{}. Our approach is summarized in Figure \ref{fig:general_scheme}.

\begin{figure}
\centerline{\includegraphics{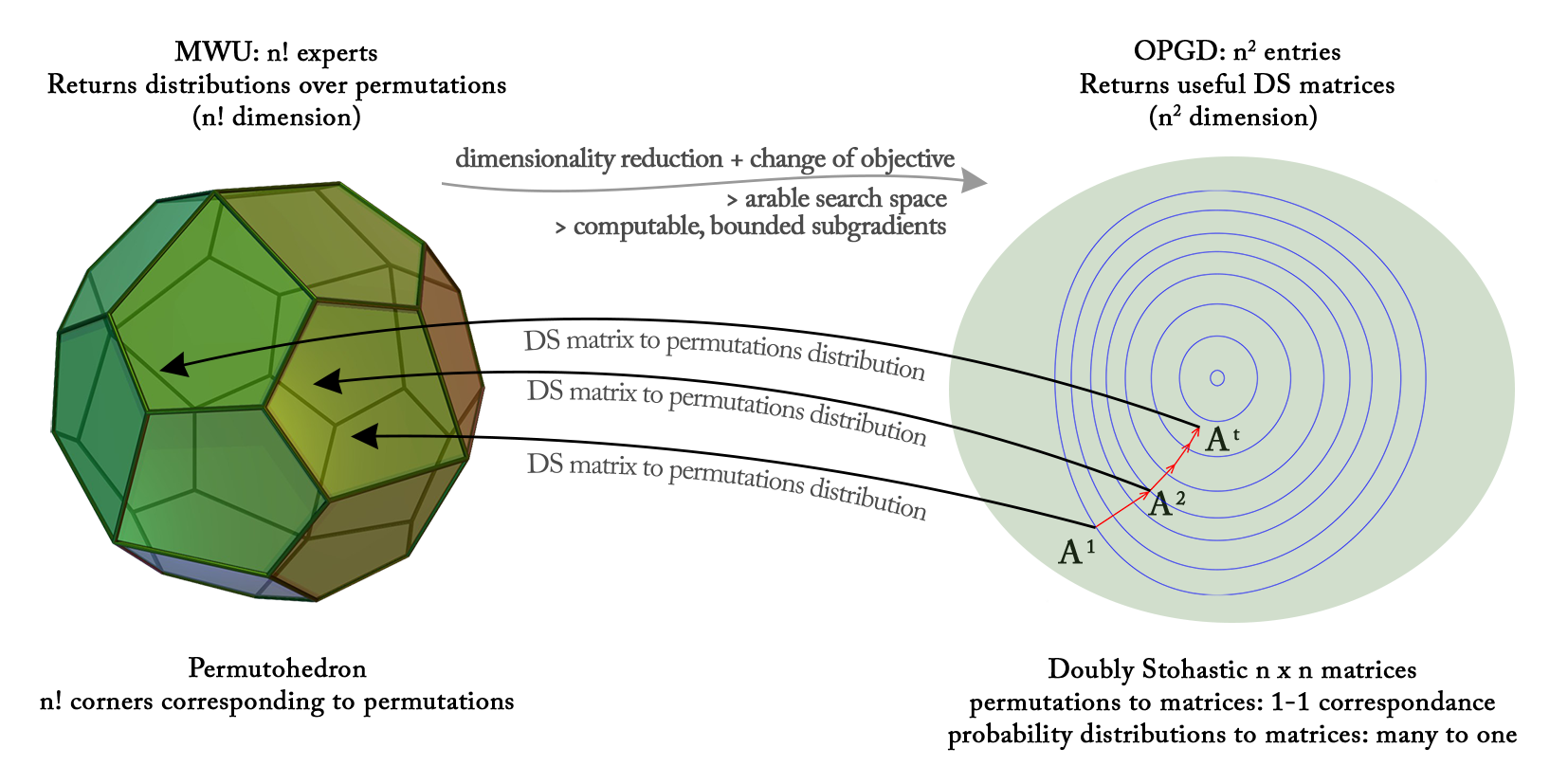}}
\caption{Our general approach, which is independent of the specific variant of \gms{}.}
\label{fig:general_scheme}
\end{figure}

\noindent\textbf{Designing and Solving the Relaxed Online Learning Problem.} 
For the relaxed version of \gms{}, we note that any permutation $\pi$ corresponds to an \emph{integral} doubly stochastic matrix $A^\pi$, with $A^\pi[i, j] = 1$ iff $\pi(j) = i$. Moreover for any request $R$, each doubly stochastic matrix is associated with a \textit{fractional access cost}. For integral doubly stochastic matrices, the fractional access cost is practically identical to the access cost of \gms{} in the respective permutation. 



The fractional access cost is given by the optimal solution of an (exponentially large) configuration linear program (LP) that relaxes \gms{} to doubly stochastic matrices (see also \cite{ImSZ14}), and is a convex function. Thus, we can use \textit{Online Projected Gradient Descent} (OPGD) \cite{Z03} to produce a \textit{no-regret} sequence of doubly stochastic matrices for the \gms{} relaxation. However, the efficient computation of the subgradient is far from trivial, due to the exponential size of the configuration LP. A key technical step is to show that the subgradient of the configuration LP can be computed in polynomial time, by solving its dual (which is of exponential size, so we resort to the elipsoid method and use an appropriate separation oracle).

%
%

\noindent\textbf{Our Results.} 
In nutshell, we resolve Question~\ref{q:main} in the affirmative. In addition to solving the relaxed version of \gms{} by a polynomial-time no-regret online learning algorithm, as described above, we present a \textit{polynomial-time randomized rounding scheme} that maps any doubly stochastic matrix to a probability distribution on permutations. The expected access cost of such a probability distribution is at most $28$ times the fractional access cost of the corresponding doubly stochastic matrix. Consequently, a $28$-regret \textit{polynomial-time randomized online learning algorithm} for \gms{} can be derived by applying, in each round, this rounding scheme to the doubly stochastic matrix $A^t$, produced by OPGD. For the important special case of \ms{}, we improve the regret bound to $11.713$ via a similar randomized rounding scheme that
exploits the fact that $\mathrm{K}(R)=1$ for all requests. 

We also present a \textit{polynomial-time deterministic rounding scheme} mapping any (possibly fractional) doubly stochastic matrix to permutations. As before, applying this scheme to the sequence of doubly stochastic matrices produced by OPGD for the relaxation of \gms{} leads to a \textit{polynomial-time deterministic online learning algorithm} with regret $2 \max_{t}|R_t|$ for \ms{}. Such a nontrivial upper bound on the regret of deterministic online learning algorithms is rather surprising.  Typically, learners that select their actions deterministically fail to achieve any nontrivial regret bounds (e.g., recall that in Learning From Expert Advice, any deterministic online algorithm has $\Omega(\# \text{experts})$ regret, which in case of \ms{} is $n!$). 
Although $2 \max_{t}|R_t|$ is not constant, one should expect that the requests are rather small in most practical applications. The above result is approximately tight, since any deterministic online learning algorithm must have regret at least $\max_{t} |R_t|/2$ \cite[Theorem~1.1]{FKKSV20}.
We should also highlight that the positive results of \cite{FKKSV20}
%
%
do not imply the existence of computationally efficient online learning algorithms for \ms{}, because their approach is based on the MWU algorithm and uses a state space of $n!$. The state of the art and our results (in bold) are summarized below.

 \begin{tabular}{||c c c c||} 
 \hline
 & Running Time & Upper Bound (Regret) & Lower Bound (Regret) \\ [0.5ex]
 \hline\hline
 \textit{GMSSC}&
Exponential ($\mathrm{MWU}$) & 1 & 1 \\
 \hline
 \textit{GMSSC}&
 Polynomial & $\boldsymbol{28}$ & 4 (any polynomial time) \\
 \hline
 \textit{MSSC}&
 Polynomial & $\boldsymbol{11.713}$ & 4 (any polynomial time)\\
 \hline
 \textit{MSSC}&
 Exponential (deterministic) & $2\cdot \max_{t}|R_t|$ & $\frac{\max_{t}|R_t|}{2}$ (any deterministic) \\
 \hline
 \textit{MSSC}&
 Polynomial (deterministic) & $ \boldsymbol{2\cdot \max_{t}|R_t|}$ & $\frac{\max_{t}|R_t|}{2}$ (any deterministic) \\
 \hline
\end{tabular}

\noindent\textbf{Related Work.} Our work relates with the long line of research concerning
the design of time-efficient online learning algorithms in various combinatorial domains in which the number of possible actions is exponentially large. Such domains include online routing \cite{HS97,AK08}, selection of permutations \cite{TW00,YHKSTT11,A14,HW07}, selection of binary search trees \cite{TM03}, submodular minimization/maximization \cite{HK12a,JB11,SG08}, matrix completion \cite{HKS12},
contextual bandits \cite{ALLS14,DHKKLRZ11} and many more.

Apart from the above line of works, concerning the design of time-efficient online learning algorithms in specific settings, another line of research studies the design of online learning algorithms considering \textit{black-box access} in offline algorithms \cite{KV03,RW17,S12,KWK10,BB06,Syr17,HK16,KKL07,FujitaHT13,Garber17,HWZ18}. In their seminal work \cite{KV03}, Kalai et al. showed how
a polynomial-time algorithm solving optimally the underlying combinatorial problem, can be converted into a  no-regret
polynomial-time online learning algorithm.
The result of Kalai et al. was subsequently improved \cite{RW17,S12,KWK10} for settings in which the underlying problem can be (optimally) solved by a specific approach, such as dynamic programming. Although there do not exist such general reductions for $\alpha$-approximation (offline) algorithms (without taking into account the combinatorial structure of each specific setting \cite{HK16}), Kakade et al. presented such a reduction for the (fairly general) class of \textit{linear optimization problems} \cite{KKL07}. Their result was subsequently improved by \cite{FujitaHT13,Garber17,HWZ18}. We remark that
the above results do not apply in our setting since \gms{} can neither be optimally solved in polynomial-time nor is a  
\textit{linear optimization problem}.

Finally our works also relates with a recent line of research
studying time-efficient online learning algorithms in settings related to selection of permutations and rankings \cite{YHKSTT11,A14,HW07,StreeterGK09,DBLP:journals/jmlr/YasutakeHTT12}. The setting considered in
\cite{YHKSTT11,A14,HW07} is very similar to \gms{}
with the difference that once request $R_t$ is revealed, the learner pays the sum of the positions of $R_t$'s elements in permutation $\pi_t$. In this case the underlying 
combinatorial optimization problem can be solved in polynomial-time meaning that the reduction of \cite{KV03} produces a time-efficient no-regret online learning algorithm. As a result, all the above works focus on improving the vanishing rate of time-average regret.
The setting considered in \cite{StreeterGK09} is based on the
\textit{submodular maximization problem}. In particular, the number of available positions is less than the number of elements, while the cost of the selected assignment
depends on the set of elements assigned to the slots (their order does not matter). Although this
problem is NP-hard, it admits an $(1-1/e)$-approximation algorithm which is matched by the presented online learning algorithm. Finally in \cite{DBLP:journals/jmlr/YasutakeHTT12}, the cost of the selected permutation is its distance from a permutation selected by the adversary. In this case the underlying combinatorial optimization problem admits an offline $11/9$-approximation algorithm, while a polynomial-time online learning algorithm with $3/2$-regret is presented. We note that \gms{} admits a fairly more complicated combinatorial structure from the above settings and this is indicated by its $4$ inapproximability result.

\section{Definitions and Notation}
\label{sec:preliminaries}

\begin{definition}[Subgradient]\label{d:subgradients}
Given a function $f:D \mapsto \Reals$, with $D \subseteq \Reals^n$, a vector $g \in \Reals^n$ is a subgradient of $f$ at point $x\in \Reals^n$, denoted $g \in \partial F(x)$, if $f(y) \geq f(x) + g^\top (y -x)~$, for all $y \in D$.
\end{definition}

A matrix $A \in [0,1]^{n \times n}$ is \emph{doubly stochastic}, if (i) $A_{ij} \geq 0$, for all $1 \leq i,j \leq n$, (ii) $\sum_{i=1}^nA_{ij}=1$, for all $1 \leq j \leq n$, and (iii) $\sum_{j=1}^nA_{ij}=1$, for all $1 \leq i \leq n$. We let $\mathrm{DS}$ denote the set of $n \times n$ doubly stochastic matrices. 

Any permutation $\pi \in [n!]$ can be represented by an \textit{integral} doubly-stochastic $A^\pi$, where $A_{ij}^\pi=1$ iff $\pi(j) = i$. Under this representation, the access cost of \gms{}, defined in (\ref{eq:access_cost}), becomes:
\begin{equation}\label{eq:acost}
\acost(\pi , R) = \sum_{i=1}^n \min \left\{ 1, {\left (\mathrm{K(R)} - \sum_{j=1}^{i-1} \sum_{e \in R} A_{ej}^\pi \right)\!\!}_+\,\,\, \right\} \,,  
\end{equation}
where we define $(x-y)_{+} = \max\{ x-y, 0 \}$.

A key notion for our algorithms and analysis is that of \emph{configurations}. Given a request $R \subset U$, a \emph{configuration} $F$ is an assignment of the elements $e \in R$ to positions $j \in [n]$ such that no two elements $e,e' \in R$ share the same position. Intuitively, a configuration wrt. a request $R$ is the set of all permutations $\pi \in [n!]$ with the elements of $R$ in the exact same positions as indicated by $F$. As a result, all permutations $\pi \in [n!]$ that agree with a configuration $F$ wrt. a request $R$ have the same $\acost(\pi,R)$. In the following, $\mathrm{F}(R)$ denotes the set of all configurations wrt. a request $R$ and $C_F$ denotes the access cost $\acost(\pi,R)$ of any permutation $\pi \in [n!]$ that agrees with the configuration $F \in \mathrm{F}(R)$. 

\begin{example}\label{e:1}
Let  $R = \{2,5,7\}$ with $\mathrm{K}(R)=2$. The configuration $F_1 =\{(2,3) , (5,1), (7,10)\}$ stands for the set of permutations $\pi \in [n!]$ in which (i) $\pi(3) = 2$, (ii) $\pi(1) = 5$, and (iii) $\pi(10) = 7$. The configuration $F_1$ is valid (i.e., $F_1 \in \mathrm{F}(R)$), because no elements of $R$ share the same position. Moreover, $\mathrm{C}_{F_1} = 3$, because any permutation $\pi$ agreeing with $F$ has cost $3$ for $\mathrm{K}(R)=2$. Similarly, for the configuration $F_2 =\{(2,3), (5,1), (7,2)\}$, $C_{F_2} = 2$.
\end{example}

\section{Solving a Relaxation of  Generalized Min-Sum Set Cover}
\label{sec:intermediate}
Next, we present an online learning problem for a relaxed version of \gms{} in the space of doubly stochastic matrices. Specifically, we consider an online learning setting where, in each round $t \geq 1$, 
\begin{enumerate}
    \item The learner selects a doubly stochastic matrix $A^t \in \mathrm{DS}$.
    
    \item The adversary selects a request $R_t$ with covering requirements $\mathrm{K}(R_t)$.
    
    \item The learner incurs the \textit{fractional access cost} $\mathrm{FAC}_{R_t}(A^t)$ presented in Definition~\ref{d:artificial_cost}.
\end{enumerate}

\begin{definition}[Fractional Access Cost]\label{d:artificial_cost} 
Given a request $R$ with covering requirements $\mathrm{K}(R)$, the fractional access cost of a doubly stochastic matrix $A$, denoted as $\mathrm{FAC}_{R}(A)$ is the value of the following linear program: 
\begin{equation}
\tag{FLP}\label{eq:ALP}
\begin{array}{lr@{}ll}
\mbox{\emph{minimize}}  & \displaystyle\sum _{F \in \mathrm{F}(R)} C_{F} \cdot y_F &\,\,+\,\, \displaystyle\frac{n^4}{\epsilon} \cdot \sum\limits_{e \in R}
\sum\limits_{j=1}^n |A_{ej} - \sum\limits_{F:(e,j) \in F} y_F| 
&\\
\mbox{\emph{subject to}}& \displaystyle\sum\limits_{\mathrm{F} \in \mathrm{F}(R)} y_F \,&= 1 \\
&y_F\, &\geq 0, \,\,\forall F \in \mathrm{F}(R)
\end{array}
\end{equation}
\end{definition}

We always assume a fixed accuracy parameter $\epsilon$ (see also Theorem~\ref{t:no-regret_artificial_cost} about the role of $\epsilon$). Hence, for simplicity, we always ignore the dependence of $\mathrm{FAC}_{R}(A)$ on $\epsilon$. We should highlight that we need to deviate from the configuration LP of \cite[Sec.~2]{ImSZ14}, because \textit{OPGD} requires an upper bound in the subgradient's norm. The $n^4$ term in \eqref{eq:ALP} was appropriately selected so as to ensure that the \textit{access cost} of the probability distribution on permutations produced by a doubly stochastic matrix is upper bounded by its \textit{fractional access cost} (see Section~\ref{sec:GMSSC}).  

An important property of the fractional access cost in  Definition~\ref{d:artificial_cost} is that for all integral doubly stochastic matrices, it is bounded from above by the access cost of \gms{} in \eqref{eq:acost}. For that, simply note that a feasible solution is setting $y_F=1$ only for the configuration that ``agrees'' in the resources of $R$ with the permutation  of the integral matrix $A$. 

\begin{corollary}\label{c:basic}
For any integral doubly stochastic matrix $A^\pi$ corresponding to a permutation $\pi \in [n!]$,
\[\mathrm{FAC}_R(A^\pi) \leq
  \acost(\pi , R).\]
\end{corollary}

For  $A^1,A^2\in \mathrm{DS}$, it is $\mathrm{FAC}_{R_t}\left(\lambda A^1 + (1-\lambda) A^2\right) \leq \lambda \cdot \mathrm{FAC}_{R_t}\left(A^1\right)
+ ( 1-\lambda) \cdot \mathrm{FAC}_{R_t}\left(A^2\right)
$, meaning that $\mathrm{FAC}_{R_t}(\cdot)$ is a convex function in the space of doubly stochastic matrices. Since doubly stochastic matrices form a convex set, \textit{Online Projected Gradient Descent} \cite{Z03} is a \textit{no-regret} online learning algorithm for the relaxed version of \gms{}.

\subsection{Implementing Online Gradient Descent in Polynomial-time}

\textit{Online Gradient Descent} requires, in each round $t$, the computation of a subgradient of the fractional access cost $\mathrm{FAC}_{R_t}(A_t)$ (see also Definition~\ref{d:subgradients}). Specifically, given a request $R$ and a doubly stochastic matrix $A$, a vector $g \in \Reals^{n^2}$ belongs to the subgradient $\partial \mathrm{FAC}_{R}(A)$, if for any 
$B\in \mathrm{DS}$,
\begin{equation}\label{eq:subgradient_doubly_stochastic}
\mathrm{FAC}_R(B) \geq \mathrm{FAC}_R(A) + g^\top (B - A)\,,
\end{equation}
where we slightly abuse the notation and think of 
matrices $A$ and $B$ as vectors in $[0,1]^{n^2}$.

Computing a subgradient $g \in \partial \mathrm{FAC}_{R}(A)$ in polynomial-time is far from trivial, because the fractional access cost $\mathrm{FAC}_{R}(A)$ does not admit a closed form, since its value is determined by the optimal solution to \eqref{eq:ALP}. Moroever, \eqref{eq:ALP} has exponentially many variables $y_F$, one for each  configuration $F \in \mathrm{F}(R)$. We next show how to compute a subgradient $g \in \partial \mathrm{FAC}_{R}(A)$ by using linear programming duality and solving the \textit{dual} of \eqref{eq:ALP}, which is presented below: 
%
\vspace{-1mm}
\begin{equation}\label{eq:dual}
\begin{array}{lr@{}ll}
\text{maximize}  & \displaystyle\lambda + \sum_{e \in R}\sum_{j=1}^n A_{ej} \cdot \lambda_{ej}
&\\
\text{subject to}& \displaystyle\lambda + \sum\limits_{(e,j)\in F}\lambda_{ej}\,&\leq C_F, \text{ for all } F \in \mathrm{F}(R)\\
&|\lambda_{ej}|\,&\leq  n^4/\eps
\end{array}
\end{equation}

\begin{lemma}\label{l:subgradients_dual}
For any request $R$ and any stochastic matrix $A\in \mathrm{DS}$, let $g \in \Reals^{n^2}$ denote the vector
consisting of the $n^2$ values of the variables $\lambda^\ast_{ej}$
in the optimal solution of \eqref{eq:dual}. Then, for any 
$B \in \mathrm{DS}$,
\[\mathrm{FAC}_R(B) \geq \mathrm{FAC}_R(A) + g^\top (B - A)\]
\noindent Moreover the Euclidean norm of $g$ is upper bounded by $n^5/\eps$, i.e., $\norm{g}_2 \leq n^5/\eps$.
\end{lemma}

Lemma~\ref{l:subgradients_dual} shows that a subgradient $g \in \partial \mathrm{FAC}_{R}(A)$ can be obtained from the solution to the dual LP \eqref{eq:dual}. Although \eqref{eq:dual} has exponentially many constraints, we can solve it in polynomial-time by the ellipsoid method, through the use of an appropriate separation oracle.%
\footnote{Interestingly, \gms{} seems to be the most general version of min-sum-set-cover-like ranking problems that allow for an efficient subgradient computation through the dual of the configuration LP \eqref{eq:ALP}. E.g., for the version of Min-Sum-Set-Cover with submodular costs considered in \cite{AG11}, determining the feasibility of a potential solution to \eqref{eq:dual} is $\mathrm{NP}$-hard. This is true even for very special case where the cover time function used in \cite{AG11} is additive.} 
In fact, our separation oracle results from a simple modification of the separation oracle in \cite[Sec.~2.3]{ImSZ14} (see also Section~\ref{appendix_online_gradient_descent}). Now, the reasons for the particular form of fractional access cost in Definition~\ref{d:artificial_cost} become clear: (i) it allows for  efficient computation of the subgradients, and (ii) the dual constraints $|\lambda_{ej}|\leq n^4/\eps$ imply that the subgradient's norm is always bounded by $n^5/\eps$.

\begin{remark}
For the \textit{Min-Sum Set Cover problem}, the use of the ellipsoid method (for the computation of the subgradient vector) can be replaced by a more efficient quadratic-time algorithm (see Appendix~\ref{sec:MSSC_random}).
\end{remark}

Having established polynomial-time computation for the subgradients,  Online Projected Gradient Descent takes the form of Algorithm~\ref{alg:projected_gradient_descent} in our specific setting.

\begin{algorithm}[H]
  \caption{Online Projected  Gradient Decent in Doubly Stochastic Matrices}\label{alg:projected_gradient_descent}
  \begin{algorithmic}[1]
  \State Initially, the player selects the matrix $A^1 = 1/n \cdot 1_{n \times n}$.
  \ForAll{ rounds $t = 1 \cdots T$}
  \State The adversary selects a request $R_t \subseteq U$ with covering requirements $\mathrm{K}(R_t)$.
  \State The learner receives cost, $\mathrm{FAC}_{R_t}(A^t)$.
  
  \State The learner computes a subgradient $g_t \in \partial \mathrm{FAC}_{R_t}(A^t)$ by solving the dual of (\ref{eq:ALP}).

  \State The learner computes the matrix, $\hat{A} = A^t - 2\eps \cdot g_t/(n^{4.5}\sqrt{t})$.
  
  \State The learner adopts the matrix, $A^{t+1} = \argmin_{A \in \mathrm{DS}} \norm{A - \hat{A}}_{\mathrm{F}}$
  \EndFor
  \end{algorithmic}
\end{algorithm}

Step~$6$ of Algorithm~\ref{alg:projected_gradient_descent} is the \textit{gradient step}. In Online Projected Gradient Descent, this step is performed with step-size $D/(G\sqrt{t})$, where $D$ and $G$ are upper bounds on the diameter of the action space and on the Euclidean norm of the subgradients. In our case, the action space is the set of doubly stochastic matrices. Since $\max_{A,B \in \mathrm{DS}}\norm{A - B}_{\mathrm{F}} \leq 2\sqrt{n}$ the parameter $D = 2\sqrt{n}$,
and $G=n^5/\eps$, by Lemma~\ref{l:subgradients_dual}. Hence, our step-size is $2\eps/(n^{4.5}\sqrt{t})$. The projection step (Step~$7$) is implemented in polynomial-time, because projecting to doubly stochastic matrices is a convex problem \cite{FJFA13}. We conclude the section by plugging in the parameters $G = n^5/\eps$ and $D=2\sqrt{n}$ to the regret bounds of Online Projected Gradient Descent \cite{Z03}, thus obtaining Theorem~\ref{t:no-regret_artificial_cost}.

\begin{theorem}\label{t:no-regret_artificial_cost}
For any $\epsilon > 0$ and any request sequence $R_1,\ldots,R_T$, the sequence of doubly stochastic matrices $A^1,\ldots,A^T$ produced by Online Projected Gradient Descent (Algorithm~\ref{alg:projected_gradient_descent}) satisfies,
$\frac{1}{T}\sum_{t=1}^T \mathrm{FAC}_{R_t}(A^t) \leq \frac{1}{T}\min_{A \in \mathrm{DS}} \sum_{t=1}^T\mathrm{FAC}_{R_t}(A) + O\left (\frac{n^{5.5}}{\epsilon \sqrt{T}}\right)$.
\end{theorem}

\section{Converting Doubly Stochastic 
Matrices to Distributions on Permutations}
\label{sec:roundings}

Next, we present polynomial-time rounding schemes that map a doubly stochastic matrix back to a probability distribution on permutations. Our schemes ensure that the resulting permutation (random or deterministic) has access cost  
at most $\alpha$ times the fractional access cost  of  the corresponding doubly stochastic matrix. Combining such schemes with  Algorithm~\ref{alg:projected_gradient_descent}, we obtain polynomial-time $\alpha$-regret online learning algorithms for \gms{}.

Due to lack of space, we only present the deterministic rounding scheme, which is intuitive and easy to explain. Most of its analysis and the description of the randomized rounding schemes are deferred to the supplementary material. 

\begin{algorithm}[H]
  \caption{Converting Doubly Stochastic Matrices to Permutations}\label{alg:drand}
  \textbf{Input:} A doubly stochastic matrix $A \in \mathrm{DS}$, a parameter $r$ and a parameter $\alpha >0$.\\
  \textbf{Output:} A deterministic permutation $\pi_A \in [n!]$.
  \begin{algorithmic}[1]
  \State $\text{Rem} \leftarrow \{1,\ldots,n\}$
  
  \For{$k = 1$ to 
  $\left \lfloor{n/r}\right \rfloor$ }
  
    \State Let $R_k$ be any $(1+\alpha)$-approximate solution to the following problem: 
    \[\min_{R \subseteq \mathrm{Rem}: |R| = r}\,\,\sum_{i=1}^n{\left (1 - \sum_{j=1}^{i-1} \sum_{e \in R} A_{ej} \right)\!\!}_+\]
    
  \State Assign the elements of $R_k$ to positions $(k-1)\cdot r + 1, \ldots, k\cdot r$ of $\pi_A$ in any order. 
  
  \State $\text{Rem} \leftarrow \text{Rem} \setminus R_k$
  
  \EndFor
  \State \textbf{return} the resulting permutation $\pi_A \in [n!]$.
  \end{algorithmic}
\end{algorithm}

 Algorithm~\ref{alg:drand} aims to produce a permutation $\pi_A \in [n!]$ from the doubly stochastic matrix $A$ such that the $\acost(\pi_A,R)$ is approximately bounded by $\mathrm{FAR}_R(A)$ for any request $R$ with $|R| \leq r$ and $\mathrm{K}(R)=1$. Algorithm~\ref{alg:drand} is based on the following intuitive greedy criterion:

\begin{quote}
Assign to the first $r$ available positions of $\pi_A$ the elements of the request of size $r$ with minimum fractional cost of Definition~\ref{d:artificial_cost} wrt. the doubly stochastic matrix $A$. Then, remove these elements and repeat.\end{quote}

Unfortunately the greedy step above involves the solution to an $\mathrm{NP}$-hard optimization problem. Nevertheless, we can approximate it with an FPTAS (Fully Polynomial-Time Approximation Scheme). The (1+$\alpha$)-approximation algorithm used in Step~$3$ of Algorithm~\ref{alg:drand} runs in $\Theta(n^4r^3/\alpha^2)$ and is presented and analyzed in Section~\ref{s:DP}. Theorem~\ref{t:det-rounding} (proved in Section~\ref{sec:MSSC_deterministic}) summarizes the guarantees on the access cost of a permutation $\pi_A$ produced by Algorithm~\ref{alg:drand}.

\begin{theorem}\label{t:det-rounding}
Let $\pi_A$ denote the permutation produced by Algorithm~\ref{alg:drand} when the doubly stochastic matrix $A$ is given as input. 
Then for any request $R$ with $\mathrm{K}(R)=1$
and $|R| \leq r$,
\[\acost(\pi_A,R) \leq 2 (1+\eps)(1+\alpha)^2 r \cdot \mathrm{FAC}_{R}(A),\]
\noindent  with $\eps>0$ 
as in Definition~\ref{d:artificial_cost}. Moreover,
Step~$3$, 
can be implemented in $\Theta(n^4r^3/\alpha^2)$ steps.
\end{theorem}

We now show how Algorithm~\ref{alg:projected_gradient_descent} and Algorithm~\ref{alg:drand} can be combined to produce a \textit{polynomial-time deterministic online learning algorithm} for \ms{} with regret roughly $2\max_{1\leq t\leq T}|R_t|$. For any adversarially selected sequence of requests  $R_1,\ldots, R_T$ with $\mathrm{K}(R_t) = 1$ and $|R_t| \leq r$,
the learner runs Algorithm~\ref{alg:projected_gradient_descent} \textit{in the background}, while at each round $t$ uses  Algorithm~\ref{alg:drand} to produce the permutation $\pi_{A^t}$ by the doubly stochastic matrix $A^t \in \mathrm{DS}$. Then,
\smallskip
\begin{eqnarray*}
\frac{1}{T}\sum_{t=1}^T \acost(\pi_{A^t},R_t)
&\leq&
\frac{1}{T} \cdot \sum_{t=1}^T 2(1+\eps)(1+\alpha)^2 r\cdot \mathrm{FAC}_{R_t}(A^t)\\
&\leq&
\frac{2r}{T}(1+\eps)(1+\alpha)^2 \cdot  \min_{A \in \mathrm{DS}}
\sum_{t=1}^T \mathrm{FAC}_{R^t}(A) + O\left(\frac{n^{5.5}}{\epsilon\sqrt{T}}\right)\\
&\leq&
\frac{2r}{T}(1+\eps)(1+\alpha)^2 \cdot
\sum_{t=1}^T \mathrm{FAC}_{R^t}(A^{\pi^\ast}) + O\left(\frac{n^{5.5}}{\epsilon\sqrt{T}}\right)\\
&\leq&
\frac{2r}{T}(1+\eps)(1+\alpha)^2 \cdot 
\sum_{t=1}^T \acost(\pi^\ast,R^t) + O\left(\frac{n^{5.5}}{\epsilon \sqrt{T}}\right)
\end{eqnarray*}
\vspace{-2mm}

The first inequality follows by Theorem~\ref{t:det-rounding}, the second by Theorem~\ref{t:no-regret_artificial_cost} and the last by Corollary~\ref{c:basic}.

Via the use of \textit{randomized rounding schemes} we can substantially improve both on the assumptions and the guarantee of Theorem~\ref{t:det-rounding}. Algorithm~\ref{alg:sampling_GMSSC} (presented in Section~\ref{sec:GMSSC}), describes such a scheme that converts any doubly stochastic matrix $A$ to a probability distribution over permutations, while Theorem~\ref{t:GMMSC-rounding} (also proven in Section~\ref{sec:GMSSC}) establishes an approximation guarantee (arbitrarily) close to $28$
on the expected access cost. 

\begin{theorem}\label{t:GMMSC-rounding}
Let $\mathrm{P_A}$ denote the probability distribution over permutations that Algorithm~\ref{alg:sampling_GMSSC} produces given as input an $A \in \mathrm{DS}$. For any request $R$,
\[\E_{\pi \sim \mathrm{P_A}}[\acost(\pi,R)] \leq 28(1+\eps) \cdot \mathrm{FAC}_{R}(A)\]
\noindent where $\eps>0$ is the parameter used in Definition~\ref{d:artificial_cost}.
\end{theorem}

Using Theorem~\ref{t:GMMSC-rounding} instead of Theorem~\ref{t:det-rounding} in the previously exhibited analysis, implies that combining Algorithms~\ref{alg:projected_gradient_descent} and ~\ref{alg:sampling_GMSSC}
leads to a \textit{polynomial-time randomized online learning algorithm} for \gms{} with $28(1+\eps)$ regret.

In Section~\ref{sec:MSSC_random} we improve Theorem~\ref{t:GMMSC-rounding} for the 
the special case of \ms{}. The randomized rounding scheme described in Algorithm~\ref{alg:sampling_MSSC} admits the approximation guarantee of Theorem~\ref{t:MMSC-rounding}, which implies a \textit{polynomial-time randomized online learning algorithm} for \ms{} with $11.713(1+\eps)$ regret

\begin{theorem}\label{t:MMSC-rounding}
Let $\mathrm{P_A}$ denote the probability distribution over permutations that Algorithm~\ref{alg:sampling_MSSC} produces given as input an $A \in \mathrm{DS}$. For any request $R$ with covering requirement $\mathrm{K}(R)=1$,
\[\E_{\pi \sim \mathrm{P}_A}[\acost(\pi,R)] \leq 11.713(1+\eps) \cdot \mathrm{FAC}_{R}(A)\]
\noindent where $\eps>0$ is the parameter used in Definition~\ref{d:artificial_cost}.
\end{theorem}
\vspace{-2mm}



\section{Experimental Evaluations}\label{s:exp}
In this section we provide experimental evaluations of all the proposed online learning algorithms (both deterministic and randomized) for  \textit{Min-Sum Set Cover}. Surprisingly enough our simulations seem to suggest that the \textit{deterministic rounding scheme} proposed in Algorithm~\ref{alg:drand}, performs significantly better than its theoretical guarantee, stated in Theorem~\ref{t:det-rounding}, that associates its regret with the cardinality of the sets.
The following figures illustrate the performance of  Algorithm~\ref{alg:drand} and Algorithm~\ref{alg:sampling_MSSC}, and compare it with the performance of the offline algorithm proposed by Feige et al. \cite{FLP04} and the performance of selecting a permutation uniformly at random at each round. In the left figure each request contains either element $1$ or $2$ and four additional randomly selected elements, while in the right figure each request contains one of the elements $\{1,2,3,4,5\}$ and nine more randomly selected elements.\footnote{In the subsequent figures the curves describing the performance of each algorithm are placed in the following top-down order i)\textit{ Selecting a permutation uniformly at random}, ii) \textit{Algorithm~\ref{alg:drand}}, iii)\textit{ Algorithm~\ref{alg:sampling_MSSC}} and iv)\text{ Feige-Lovasz-Tetali algorithm}~\cite{FLP04}.} We remark that in our experimental evaluations, we solve the optimization problem of Step~$3$ in Algorithm~\ref{alg:drand} through a simple heuristic that we present in Appendix~\ref{s:heuristic}, while for the computation of the subgradients we use the formula presented in Corollary~\ref{c:gradient}. The code used for the presented simulations can be found at \url{https://github.com/sskoul/ID2216}.

\begin{figure}[ht]
  \centering
  \includegraphics[valign = c, width = 6.5cm]{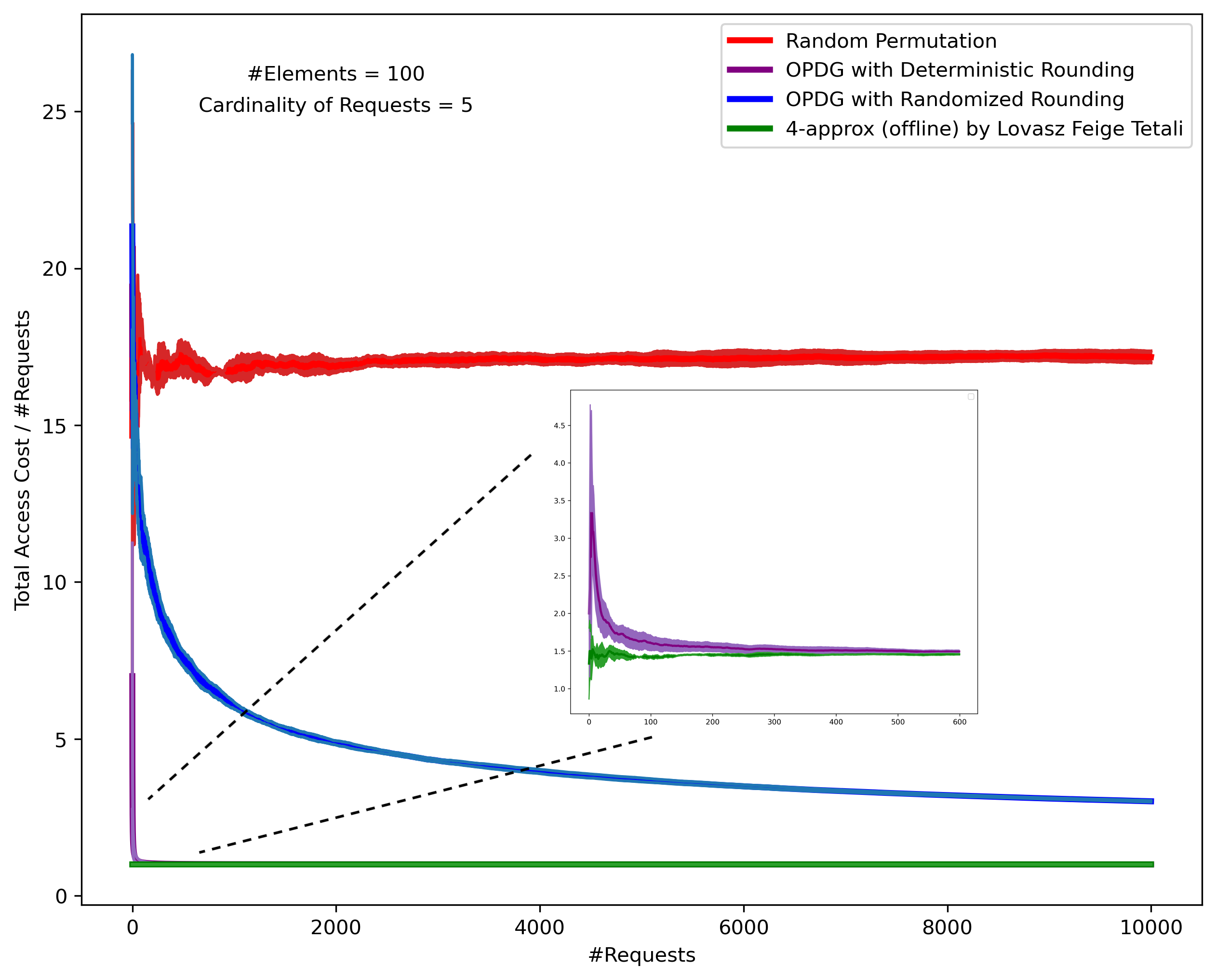}
  \qquad
  \includegraphics[valign = c, width = 6.5cm]{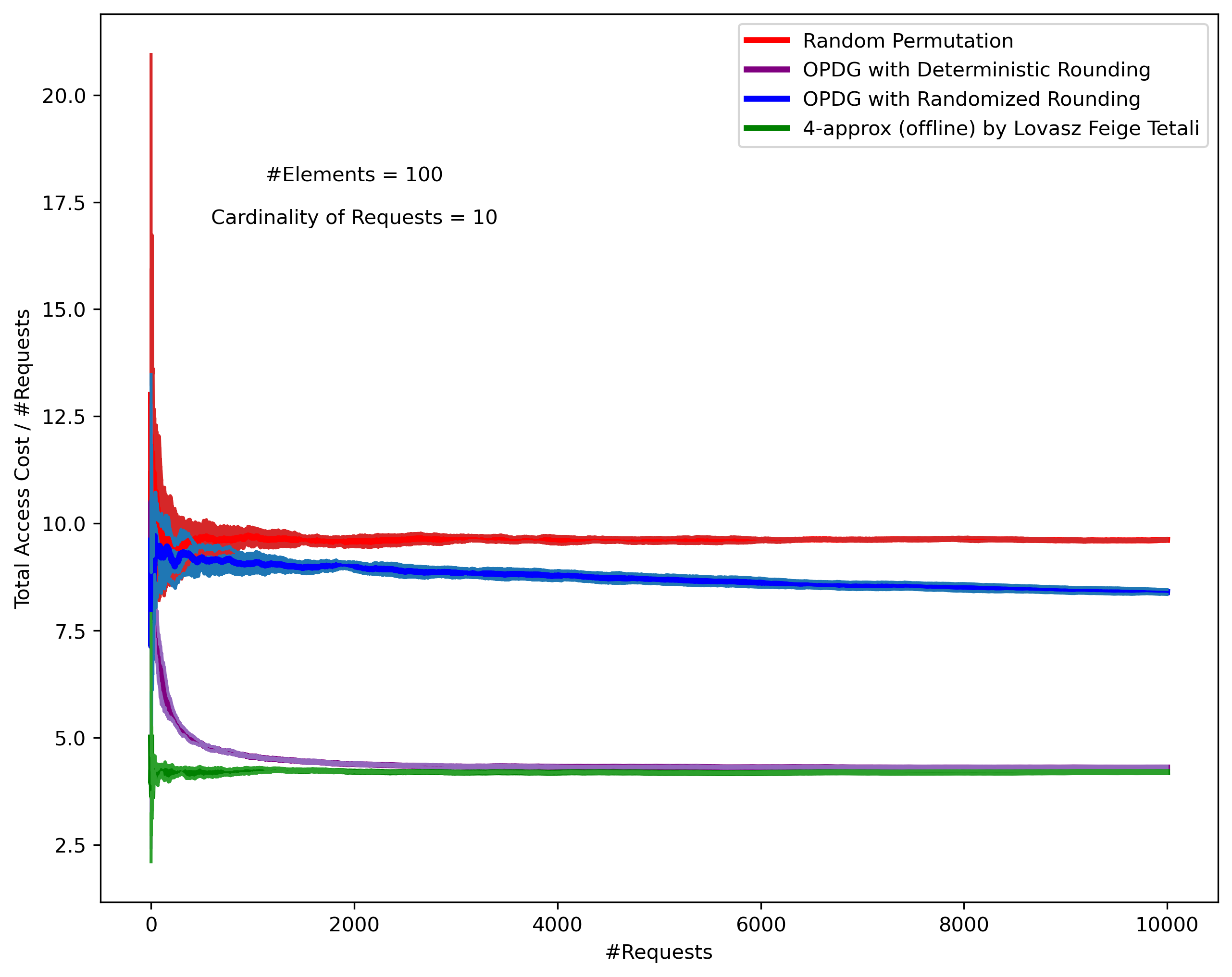}
  \qquad
\end{figure}

\section{Conclusion}
This work examines polynomial-time online learning algorithms for  \textit{(Generalized) Min-Sum Set Cover}. Our results are based on solving a relaxed online learning problem of smaller dimension via \textit{Online Projected Gradient Descent}, the solution of which is transformed at each round into a solution of the initial action space with bounded increase in the cost. To do so, the cost function of the relaxed online learning problem is defined by the value
of a linear program with exponentially many constraints. Despite its exponential size, we show that the subgradients can be efficiently computed via associating them with the variables of the LP' s \textit{dual}. We believe that the bridge between online learning algorithms (e.g. online projected gradient descent) and traditional algorithmic tools (e.g. duality, separation oracles, deterministic/randomized rounding schemes), introduced in this work, is a promising new framework for the design of efficient online learning algorithms in high dimensional combinatorial domains.
Finally closing the gap between our regret bounds and the lower bound of  $4$, which holds for 
polynomial-time online learning algorithms for \ms{}, is an interesting open problem.

\section*{Broader Impact}
We are living in a world of abundance, where each individual is provided myriad of options in terms of available products
and services (e.g. music selection, movies etc.). Unfortunately this overabundance makes the cost of exploring all of them prohibitively large. This problem is only compounded by the fast turn around of new trends at a seemingly ever increasing rate.
Our algorithmic techniques provide a practically applicable methodology for managing this complexity.

\section*{Funding Disclosure}
Dimitris Fotakis and Thanasis Lianeas are supported by the Hellenic Foundation for Research and Innovation (H.F.R.I.) under the ``First Call
for H.F.R.I. Research Projects to support Faculty members and Researchers' and the procurement of high-cost research equipment grant'', project BALSAM, HFRI-FM17-1424. Stratis Skoulakis was
supported by NRF 2018 Fellowship NRF-NRFF2018-07.
 G.~Piliouras gratefully acknowledges AcRF Tier-2 grant (Ministry of Education – Singapore) 2016-T2-1-170, grant PIE-SGP-AI-2018-01, NRF2019-NRF-ANR095 ALIAS grant and NRF 2018 Fellowship NRF-NRFF2018-07 (National Research Foundation Singapore).
 
\bibliographystyle{plain}
\bibliography{refs}


\newpage

\begin{appendices}
\section{Omitted Proofs of Section~\ref{sec:intermediate}}
\label{appendix_online_gradient_descent}
\begin{proof}[Proof of Lemma~\ref{l:subgradients_dual}]
To simplify notation, let $\lambda^\ast(A),\lambda^\ast_{ej}(A)$ denote the values of the variables $\lambda^\ast(A),\lambda^\ast(A)_{ej}$ in the optimal solution of the dual program written with respect to doubly stochastic matrix $A \in \mathrm{DS}$. Respectively $\lambda^\ast(B),\lambda_{ej}^\ast(B)$ for the doubly stochastic matrix $B \in \mathrm{DS}$. By strong duality, we have that
\[\mathrm{FAC}_R(A) = \lambda^\ast(A) + \sum_{e\in R}\sum_{j=1}^n A_{ej} 
\cdot \lambda^\ast_{ej}(A) \text{ and }
\mathrm{FAC}_R(B) = \lambda^\ast(B) + \sum_{e\in R}\sum_{j=1}^n B_{ej} 
\cdot \lambda^\ast_{ej}(B)
\]
\noindent Since  matrices $A$ and $B$ only affect the objective function of the dual and not its constraints, the solution $\lambda^\ast(A),\lambda_{ej}^\ast(A)$ is a feasible solution for the dual program written according to matrix $B$. By the optimality of $\lambda^\ast(B),\lambda^\ast_{ej}(B)$ we get,
\[\mathrm{FAC}_R(B) = \lambda^\ast(B) + \sum_{e\in R}\sum_{j=1}^n B_{ej} 
\cdot \lambda^\ast_{ej}(B) \geq \lambda^\ast(A) + \sum_{e\in R}\sum_{j=1}^n B_{ej} 
\cdot \lambda^\ast_{ej}(A)\]
\noindent As a result, we get that
$\mathrm{FAC}_R(B) - \mathrm{FAC}_R(A)\geq \sum_{e\in R}\sum_{j=1}^n \lambda^\ast_{ej}(A) \cdot(B_{ej} - A_{ej}) 
$ implying that the vector $g$ containing the $\lambda^\ast_{ej}(A)$'s, is a subgradient of $\mathrm{FAC}_R(\cdot)$ at point $A$, i.e., $g \in \partial \mathrm{FAC}_{R}(A)$. The inequality $\norm{g}_2 \leq n^5/\eps$ directly follows by the fact that $|\lambda^\ast(A)_{ej}| \leq n^4/\eps$.
\end{proof}

\noindent
\textit{Separation Oracle for the LP in Equation~\ref{eq:dual}:
} 
The dual linear program of \eqref{eq:dual} is differs from the $\text{LP}_{\text{dual}}$ in \cite[Sec.~2.2]{ImSZ14} only in the constraints $|\lambda_{ej}|\leq n^4 /\eps$, which are only present in \eqref{eq:dual}. \cite[Sec.~2.2]{ImSZ14} present a separation oracle for their $\text{LP}_{\text{dual}}$ (i.e., for \eqref{eq:dual}, without the constraints $|\lambda_{ej}|\leq n^4/\eps$), which is based on formulating and solving a min-cost flow problem. Since, in case of \eqref{eq:dual}, the we have only $n^2$ additional constraints $|\lambda_{ej}|\leq n^4/\eps$, we can first check whether these constraints are satisfied by the current solution and then run the separation oracle of \cite{ImSZ14}.

\section{Omitted Proofs of Section~\ref{sec:roundings}}

\subsection{Proof of Theorem~\ref{t:GMMSC-rounding}}
\label{sec:GMSSC}
\noindent In Algorithm~\ref{alg:sampling_GMSSC}, we present the \textit{online randomized rounding scheme} that combined with \textit{Projected Gradient Descent} (Algorithm~\ref{alg:projected_gradient_descent}) produces a \textit{polynomial-time randomized online learning algorithm}
for \gms{} with (roughly) $28$ regret. The randomized rounding scheme described in Algorithm~\ref{alg:sampling_GMSSC} was introduced by \cite{SW11} to provide a $28$-approximation algorithm for the (offline) \gms{}. \cite{SW11} proved that this randomized rounding scheme produces a random permutation with access cost  
at most $28$ times greater than the optimal fractional value of the LP relaxation of \gms{} introduced in \cite{BGK10}. We remark that this LP relaxation cannot be \textit{translated} to an equivalent \textit{relaxed online learning problem} as the one we formulated using the \textit{fractional access cost} of Definition~\ref{d:artificial_cost}. The goal of the section is to prove Theorem~\ref{t:GMMSC-rounding} which extends the result of \cite{SW11} to the \textit{fractional access cost} of Definition~\ref{d:artificial_cost}.

\begin{algorithm}[H]
  \caption{Converting Doubly Stochastic Matrices to Probability Distributions over Permutations}\label{alg:sampling_GMSSC}
  \textbf{Input:} A doubly stochastic matrix $A \in \mathrm{DS}$.\\
  \textbf{Output:} A probability distribution over permutations, $\mathrm{P}_A \sim \pi \in [n!]$ 
  \begin{algorithmic}[1]
  \State Randomly pick $\alpha \in (0,1)$ with probability density function $f(\alpha) = 2\alpha$.
  \State Set  $B \leftarrow (5.03/\alpha) \cdot A$ 
  \For{ all elements $e = 1$ to $n$ }
  \For{ all positions $j = 1$ to $\left \lfloor{n/2}\right \rfloor$ }
    \State $B_{e,2j} \leftarrow B_{e,2j} + B_{e,j}$.
  \EndFor
  \EndFor
  \For{ all elements $e = 1$ to $n$ }
    \State  Pick $\alpha_e$ uniformly at random in $[0,1]$.
    \State  Find the effecive index $i_e^\alpha \leftarrow \argmax_{i} \{i: \sum_{j=1}^{i-1}B_{ej} < \alpha_e\}$.
  \EndFor
  \State Output the elements according to the order of $i_e$'s.
    \end{algorithmic}
\end{algorithm}

\begin{definition}\label{d:cost_SW}
For a request $R$ with covering requirements $\mathrm{K}(R)$, we define the cost $\mathrm{SW}_R:\mathrm{DS} \mapsto \Reals$ on the doubly stochastic matrices as follows: For any doubly stochastic matrix $A \in \mathrm{DS}$, the value $\mathrm{SW}_R(A)$ equals the value of the following linear program,
\begin{equation*}
\begin{array}{ll@{}ll}
\text{minimize}  & \displaystyle\sum _{i=1}^n \left(1-z_i \right)&\\
\text{subject to}& \displaystyle \left(\mathrm{K}(R) - |M| \right)\cdot z_i \leq \sum\limits_{j=1}^{i-1}\sum\limits_{e \in R\setminus M}A_{ej}~\text{  for all }~ M \subseteq R\\
&z_i \in [0,1]~\text{ for all}~ 1\leq i\leq n
\end{array}
\end{equation*}
\end{definition}

\begin{lemma}\cite{SW11}\label{l:skutella}
For any doubly stochastic matrix $A \in \mathrm{DS}$, 
\[\E_{\pi \sim \mathrm{P}_A}[\acost(\pi,R)] \leq 28 \cdot \mathrm{SW}_{R}(A)\]
where $\mathrm{P}_A$ is the probability distribution over the permutation produced by Algorithm~\ref{alg:sampling_GMSSC} when the matrix $A$ was given as input.
\end{lemma}

In Lemma~\ref{l:comparison} we  associate the cost $\mathrm{SW}_R(\cdot)$ of Definition~\ref{d:cost_SW} with the \textit{fractional access cost}
$\mathrm{FAC}_{R}(\cdot)$ of Definition~\ref{d:artificial_cost}. Then Theorem~\ref{t:GMMSC-rounding} directly follows by Lemma~\ref{l:skutella} and Lemma~\ref{l:comparison}.

\begin{lemma}\label{l:comparison}
For any doubly stochastic matrix $A \in \mathrm{DS}$, \[\mathrm{SW}_R(A) \leq  (1+\epsilon)\cdot \mathrm{FAC}_R(A)\]
where $\eps>0$ is the parameter of the linear program (\ref{eq:ALP}) in Definition~\ref{d:artificial_cost}.
\end{lemma}
\begin{proof}
Starting from the optimal solution $y_F$ of the linear program (\ref{eq:ALP}) of $\mathrm{FAC}_R(A)$ in Definition~\ref{d:artificial_cost},
we construct a feasible solution for the linear program of $\mathrm{SW}_R(A)$ of Definition~\ref{d:cost_SW} with cost approximately bounded by $(1+\eps)\cdot \mathrm{FAC}_R(A)$.
We first prove Claim~\ref{c:1} that is crucial for the subsequent analysis.

\begin{claim}\label{c:1}
For any element $e\in R$ and position $1\leq j \leq n$,
$|A_{ej} -\sum\limits_{F:(e,je,e,j)\in F}y_F^\ast| \leq \eps/n^3$.
\end{claim}
\begin{proof}
Since $A$ is a doubly stochastic matrix,  by the Birkhoff-von Neumann theorem there exists a vector $\hat{y}$ with $\hat{y}_F \geq 0$ and $\sum_{F \in \mathrm{F}(R)}\hat{y}_F = 1$
such that \[|A_{ej} -\sum\limits_{F:(e,j)\in F}\hat{y}_F|=0 \text{ for all }e \in R\text{ and } 1\leq j \leq n \]

Since $y^\ast$ is the optimal solution, we have that
\[\sum_{F \in \mathrm{F}(R)}C_F \cdot y_F^\ast + \frac{n^4}{\epsilon}\cdot\sum_{e \in R}\sum_{j=1}^n|A_{ej} -\sum\limits_{F:(e,j)\in F}y_F^\ast|
\leq \sum_{F \in \mathrm{F}(R)}C_F \cdot \hat{y}_F.
\]
\noindent Now the claim follows by the fact that $1 \leq C_F \leq n$, $\hat{y}_F \geq 0$ and $\sum_{F \in \mathrm{F}(R)}\hat{y}_F = 1$.
\end{proof}
\noindent Having established Claim~\ref{c:1}, we construct the
solution $z^\ast$ that is feasible for the linear program of Definition~\ref{d:cost_SW} and its value (under the linear program of Definition~\ref{d:cost_SW}), is upper bounded by
$(1+\eps)\cdot \mathrm{FAC}_R(A)$. For each position $1 \leq i \leq n$,
\[z_i^\ast = \left(\sum_{F \in \mathrm{F}(R): C_F \leq i-1}y_F^\ast - \frac{\epsilon}{n}\right)_+\]

We first prove that $z^\ast$ is feasible for the linear program of Definition~\ref{d:cost_SW}. At first observe that in case $z_i^\ast =0$ or $\mathrm{K}(R) - |M| \leq 0$ for some $M \subseteq R$, the constraint  $\left(\mathrm{K}(R) - |M| \right)\cdot z_i \leq \sum_{j=1}^{i-1}\sum_{e \in R \setminus M}A_{ej}$ is trivially satisfied. We thus turn our attention in the cases where $z_i^\ast = \sum_{F: C_F \leq i-1}y_F^\ast - \epsilon/n >0$ and $K(R) - |M| \geq 1$ (recall, $K(R)$ and $|M|$ are integers). Applying Claim~\ref{c:1} we get that, 
\begin{eqnarray*}
\sum_{e \in R \setminus M} \sum_{j=1}^{i-1}A_{ej} &\geq& \sum_{e \in R \setminus M}\sum_{j=1}^{i-1}\left( \sum_{F:(e,j)\in F}y_F^\ast - \epsilon/n^3 \right)\\
&\geq& \sum_{e \in R \setminus M}\sum_{j=1}^{i-1}\sum_{F:(e,j)\in F}y_F^\ast - \epsilon/n\\
&=& \sum_{F \in \mathrm{F}(R)} y_F^\ast \sum_{e \in R \setminus M}\sum_{j=1}^{i-1} \textbf{1}[(e,j) \in F] - \epsilon/n\\
&\geq& \sum_{F : C_F <i} y_F^\ast \sum_{e \in R \setminus M}\sum_{j=1}^{i-1} \textbf{1}[(e,j) \in F] - \epsilon/n\\
&\geq& \left(K(R) - |M|\right) \sum_{F : C_F <i} y_F^\ast - \epsilon/n\\
&=& \left(K(R) - |M|\right)\cdot z_i^\ast + \epsilon \frac{K(R)-|M|}{n} - \epsilon/n\\
&\geq& \left(K(R) - |M|\right)\cdot z_i^\ast
\end{eqnarray*}
\noindent where the second to last inequality follows from $C_F<i$, and the last  equation and the last inequality  follow from  $z_i^\ast + \eps /n = \sum_{F:C_F \leq i-1}y_F^\ast$ and  $\mathrm{K}(R) - |M| \geq 1$, respectively.

We complete the proof of Lemma~\ref{l:comparison} by showing that $\sum_{i=1}^n(1 - z_i^\ast) \leq (1+\eps)\cdot \mathrm{FAC}_R(A)$.
\begin{eqnarray*}
\mathrm{SW}_R(A) &\leq& \sum_{i=1}^n \left(1 - z_i^\ast\right)\\
&\leq& \sum_{i=1}^n \left(1 - \sum_{F:C_F <i}y_F^\ast + \eps/n\right)\\
&=& \sum_{i=1}^n \left(1 - \sum_{F:C_F <i} y_F^\ast\right) + \eps\\
&=& \sum_{i=1}^n \sum_{F:C_F \geq i}y_F^\ast + \eps\\
&=& \sum_{F \in \mathrm{F}(R)}C_F \cdot y_F^\ast + \eps\\
&\leq& (1+\epsilon)\cdot \mathrm{FAC}_R(A)
\end{eqnarray*}
\end{proof}

\subsection{Proof of Theorem~\ref{t:MMSC-rounding}}
\label{sec:MSSC_random}
We first present the online sampling scheme, described in Algorithm~\ref{alg:sampling_MSSC}, that produces the $11.713$ guarantee of Theorem~\ref{t:MMSC-rounding}.

\begin{algorithm}[H]
  \caption{Converting Doubly Stochastic Matrices to Probability Distribution (the case of MSSC)}\label{alg:sampling_MSSC}
  \textbf{Input:} A doubly stochastic matrix $A \in \mathrm{DS}$.\\
  \textbf{Output:} A probability distribution over permutations, $\mathrm{P}_A \sim \pi \in [n!]$.
  \begin{algorithmic}[1]

  \State Randomly pick $\alpha \in (0,1)$ with probability density function $f(\alpha) = 2\alpha$.
  
  \State Set $B\leftarrow Q \cdot A$ where  $Q\leftarrow 1.6783/\alpha$.

  \For{ all elements $e = 1$ to $n$ }
  \For{ all positions $j = 1$ to $\left \lfloor{n/2}\right \rfloor$ }
    \State $B_{e,2j} \leftarrow B_{e,2j} + B_{e,j}$
    
  \EndFor
  \EndFor
  
  \For{ all elements $e = 1$ to $n$ }
    \State  Pick $\alpha_e$ uniformly at random in $[0,1]$.
    
    \State  $i_e \leftarrow \max\{i: \sum_{j=1}^{i-1}B_{ej} < \alpha_e\}$
  \EndFor

  \State Output the elements according to the order of $i_e$'s. 
  
    \end{algorithmic}
\end{algorithm}

\noindent 
We dedicate the rest of the section to prove Theorem~\ref{t:MMSC-rounding}.
Notice that Algorithm~\ref{alg:sampling_MSSC}
is identical to Algorithm~\ref{alg:sampling_GMSSC} with a slight difference in Step~$2$. Taking advantage of $K(R)=1$, 
with tailored analysis, 
we significantly improve to $11.713$ the $28$ bound of Lemma~\ref{l:skutella}. 
Once Lemma~\ref{l:MSSC} below is established,  Theorem~\ref{t:MMSC-rounding} follows by the exact same steps that 
Theorem~\ref{t:GMMSC-rounding} follows using Lemma~\ref{l:skutella}. The proof of Lemma~\ref{l:MSSC} is concluded at the end of the section.

\begin{lemma}\label{l:MSSC}
Let $\mathrm{P}_A$ denote the probability distribution over permutations produced by Algorithm~\ref{alg:sampling_MSSC} when matrix $A$ is given as input. For all requests $R$ with $\mathrm{K}(R) = 1$,
\[\E_{\pi \sim \mathrm{P}_A}[\acost(\pi,R)] \leq 11.713 \cdot \mathrm{SW}_{R}(A)\]
where $\mathrm{SW}_{R}(\cdot)$ is the cost of Definition~\ref{d:cost_SW}.
\end{lemma}


In fact $\mathrm{SW}_{R}(\cdot)$ takes a simpler form.
\begin{corollary}\label{cor:2}
For any request $R$ with covering requirement $\mathrm{K}(R)=1$, the cost $\mathrm{SW}_R(\cdot)$ of Definition~\ref{d:cost_SW} takes the following simpler form,
\[\mathrm{SW}_R(A) = \sum_{i=1}^n  \left(1 - \sum_{j=1}^{i-1}\sum_{e \in R}A_{ej}\right)_+\]
\end{corollary}

\begin{lemma}\cite{SW11}\label{l:constructed_matrix}
For the matrix $B$ constructed at Step~$2$ of Algorithm~\ref{alg:sampling_MSSC}, the following holds:

\begin{enumerate}
    \item $\sum\limits_{j=1}^{2^k i}B_{ej} \geq (k+1) \sum\limits_{j=1}^{i} A_{ej}$
    
    \item $\sum\limits_{j=1}^{i}\sum\limits_{e=1}^n B_{ej} \leq 2 Q \cdot i$.
\end{enumerate}
\end{lemma}

Condition~$2$ of Lemma~\ref{l:constructed_matrix} allows for a bound on  the expected access cost of the  probability distribution produced by Algorithm~\ref{alg:sampling_MSSC} with respect to the indices $i_e$ 
of Step~$10$. This is formally stated below. 

\begin{lemma}\label{l:access-cost}
Let $\mathrm{P}^A_\alpha$ denote the probability distribution produced in Steps~$2-11$ of Algorithm~\ref{alg:sampling_MSSC} for a fixed value of $\alpha$. Then for any request $R$ with covering requirements $\mathrm{K}(R) = 1$,

\[ \E_{\pi \sim \mathrm{P_\alpha^A}}[\acost(\pi,R)] \leq 2Q \cdot \E[\min_{e \in R}i_{e}] + 1,\]
with $i_e$ as defined  in Step~$10$ of Algorithm~\ref{alg:sampling_MSSC}.
\end{lemma}
\begin{proof}
Let $O_i^R$ denote the set of elements outside $R$ with
index value $i_e \leq i$,
\[O_i^R =\{e \notin R:~ i_e \leq i\}.\] 
\noindent Notice that Algorithm~\ref{alg:sampling_MSSC}
orders the elements with respect to the values $i_e$ (Step~$12$). Since the covering requirements of the request $R$ is $\mathrm{K}(R)=1$, \[\acost(\pi,R) \leq |O_{\min_{e \in R}i_e}^R| + 1.\] 
The latter holds since $R$ is covered at the first index in which one of its elements appears ($\mathrm{K}(R) = 1$).
As a result,
\[ \E_{\pi \sim P^A_\alpha}[\acost(\pi,R)] \leq \E[|O_{\min_{e \in R} i_e }^R|] + 1
\leq \sum_{e' \notin R}\Pr[i_{e'} \leq \min_{e \in R} i_e] +1
\]
\noindent It is not hard to see that,
\[ \sum_{e' \notin R}\Pr[i_{e'} \leq \min_{e \in R} i_e] +1 = \E[\sum_{e'\notin R}\sum_{j=1}^{\min_{e \in R} i_e }B_{e'j}]+1
\leq 2Q \cdot \E[\min_{e \in R} i_e] + 1
\]
where the first equality follows by the fact that, once $B$ is fixed, $\Pr[i_e \leq k] = \sum_{j=1}^k B_{ej}$ (Step~10 of Algorithm~\ref{alg:sampling_MSSC}) and the last inequality follows by Case~$2$ of Lemma~\ref{l:constructed_matrix}.
\end{proof}

\begin{lemma}\label{l:integral}
Let $i_{R}^\alpha$ denote the first position at which $\sum\limits_{j=1}^{i_{R}^\alpha}\sum\limits_{e \in R} A_{ej}\geq \alpha$ then
\[\int_{0}^1i_{R}^\alpha ~ d \alpha \leq \sum_{i=1}^n  \left(
1 - \sum_{j=1}^{i-1}\sum_{e \in R}A_{ej}\right)_+ =\mathrm{SW}_R(A) \]
\end{lemma}

\begin{proof}
In order to prove Lemma~\ref{l:integral}, let us assume that a random variable $\beta$ is selected according to the uniform probability distribution in $[0,1]$, i.e., with density function  $f(\beta) =1$. 
As a result, $\int_{0}^1i_{R}^\alpha~ d\alpha=\int_{0}^1i_{R}^\beta~ d\beta = \E[i_{R}^\beta] = \sum
\limits_{i=1}^n\Pr[i_{R}^\beta \geq i]$. 
Since $i_{R}^\beta$ is the first position at which $\sum\limits_{j=1}^{i_{R}^\beta}\sum\limits_{e \in R} A_{ej}\geq \beta$,
\[\Pr[i^\beta_{R} \geq i] = \Pr[\beta > \sum_{j=1}^{i-1}\sum_{e\in R}A_{ej}]  = \max \left( 
1 - \sum_{j=1}^{i-1}\sum_{e\in S}A_{ej},0\right)\leq \sum_{i=1}^n \left(
1 - \sum_{j=1}^{i-1}\sum_{e \in R}A_{ej}\right)_+\] 
with the second equality following because $\beta$ is selected according to the uniform distribution in $[0,1]$.
\end{proof}


To this end we have upper bounded the expected access cost of Algorithm~\ref{alg:sampling_MSSC} by  $\E[\min_{e \in R}i_e]$ (Lemma~\ref{alg:sampling_MSSC}) and lower bounded 
$\mathrm{SW}_R(A)$  by $\int_{0}^1i_{R}^\alpha ~ d \alpha$ (Lemma~\ref{l:integral}). In Lemma~\ref{l:dn_kserw} we associate these bounds. 
At this point the role of Condition~$1$ of Lemma~\ref{l:constructed_matrix} is revealed.

\begin{lemma}\label{l:dn_kserw}
Let $i_R^\alpha$ denote the first position at which $\sum\limits_{j=1}^{i_{R}^\alpha}\sum\limits_{e \in R} A_{ej}\geq \alpha$ then
\[\E[\min_{e \in R}i_e] \leq i_{R}^\alpha/ (1 - 2e^{-\alpha Q}).\]
\end{lemma}
\begin{proof}
\begin{eqnarray*}
\Pr[\min_{e \in R}i_e\geq 2^k \cdot i_{R}^\alpha + 1] &=& \Pi_{e \in R}
\Pr[i_e \geq 2^k \cdot i_{R}^\alpha + 1]\\
&=& \Pi_{e \in R}\Pr[\alpha_e > \sum\limits_{j=1}^{2^k\cdot i_{R}^\alpha}B_{ej}]\\
&=& \Pi_{e \in R}\left( 1 - \sum_{j=1}^{2^k  \cdot i_{R}^\alpha} B_{ej}\right)_+\\
&\leq& e^{- \sum\limits_{e \in R}\sum\limits_{j=1}^{2^k \cdot i_{R}^\alpha} B_{ej}}\\
&\leq& e^{-(k+1)Q\sum\limits_{e \in S}\sum\limits_{j=1}^{i_{R}^\alpha} A_{ej}}\\
&\leq& e^{- (k+1)Q\alpha} = p^{k+1}
\end{eqnarray*}
where the second inequality follows by Case~$1$ of Lemma~\ref{l:constructed_matrix} and the definition (and manipulation) of matrix $B$ inside Algorithm~\ref{alg:sampling_MSSC}.
\begin{eqnarray*}
\E[\min_{e \in R}i_e] &=& i_{R}^\alpha + \sum_{k=1}^\infty \Pr[2^{k-1}\cdot i_{R}^\alpha + 1\leq i_R \leq 2^k\cdot i_{R}^\alpha]\cdot 2^k \cdot i_{R}^\alpha\\
&\leq& i_{R}^\alpha + \sum_{k=1}^\infty 2^k \cdot i_{R}^\alpha \cdot e^{-kQ\alpha} =   i_{R}^\alpha/(1-2e^{-Q\alpha})
\end{eqnarray*}
\end{proof}

\begin{lemma}\label{l:final}
Let $Q := z /\alpha$ for some positive constant $z$. For any request $R$ with covering requirement $\mathrm{K}(R)=1$,
\[\E_{\pi \sim \mathrm{P}^A}\left[\acost(\pi,R)\right] 
\leq \left( \frac{4z}{1 - 2e^{-z}} + 1\right) \cdot 
\sum_{i=1}^n \left(
1 - \sum_{j=1}^{i-1}\sum_{e\in R}A_{ej}\right)_+
\]
\end{lemma}
\begin{proof}
\begin{eqnarray*}
\E_{\pi \sim \mathrm{P}^A}[ \acost(\pi,R)]&=& 
\int_{0}^1 \E_{\pi \sim \mathrm{P}^A_\alpha}[ \acost(\pi,R)] \cdot (2\alpha)~ d \alpha \\ 
&\leq&
\int_{0}^1 2Q\cdot \E[\min_{e \in R}i_e]\cdot (2\alpha)~ d \alpha + \int_{0}^1(2\alpha) ~d \alpha~~(Lemma~\ref{l:access-cost})\\
&=& \int_{0}^1 4z\cdot \E[\min_{e \in R}i_e] d \alpha + 1~~(Q=z/\alpha)\\
&\leq& \int_{0}^1 4z \cdot i_{R}^\alpha / (1 - 2e^{-z}) ~d \alpha + 1~~(Lemma~\ref{l:dn_kserw}\mbox{ and }Q=z/\alpha)\\
&=& \frac{4z}{1 - 2e^{-z}} \int_0^1 i_{R}^\alpha ~d\alpha + 1~~(z=\alpha Q\mbox{ is constant})\\
&\leq& \left( \frac{4z}{1 - 2e^{-z}} + 1 \right ) \sum_{i=1}^n \left(
1 - \sum_{j=1}^{i-1}\sum\limits_{e \in S}A_{ej}\right)_+~~(Lemma~\ref{l:dn_kserw})\\
&=& \left( \frac{4z}{1 - 2e^{-z}} + 1 \right )\mathrm{SW}_R(A)~~(Corollary~\ref{cor:2})
\end{eqnarray*}
\end{proof}
Lemma~\ref{l:MSSC} directly follows by setting $z := 1.6783$ in Lemma~\ref{l:final}.

We conclude the section with the following corollary that provides with a quadratic-time algorithm for computing the subgradient in case of \textit{Min-Sum Set Cover problem}.

\begin{corollary}\label{c:gradient}
Let a doubly stochastic matrix $A$ and a request $R$. Let $i^\ast$ denotes the index at which $\sum_{j=1}^{i-1}A_{ej}\leq 1$ and $\sum_{j=1}^{i-1}A_{ej}> 1$. Let also the $n \times n$ matrix $B$ defined as follows,
\[ B_{ej}= \left\{
\begin{array}{ll}
      i^\ast - j & \text{ if } j \leq i^\ast - 1 \text{ and } e\in R\\
      0& \text{otherwise}\\
\end{array} 
\right. \]
The matrix $B$ (vectorized) is a subgradient of the $SW_R(\cdot)$ at point $A$.
\end{corollary}

\subsection{Proof of Theorem~\ref{t:det-rounding}}
\label{sec:MSSC_deterministic}

  
  
    
  
  

All steps of Algorithm~\ref{alg:drand} run in polynomial-time. In Step~$3$ of Algorithm~\ref{alg:drand}, any $(1+\alpha)$-approximation, polynomial-time algorithm for $\min_{R \in [\mathrm{Rem}]^r}\acost(R,A)$  
can be used. The first choice that comes in mind is  exhaustive search over all the requests of size $r$, resulting in $\Theta(n^{r})$ time complexity. Since the latter is not polynomial, we provide a $(1+\alpha)$-approximation algorithm running in polynomial-time in both parameters $n$ and $r$. For clarity of exposition the algorithm used in Step~$3$ is presented in Section~\ref{s:DP}. 
In the following we focus on proving Theorem~\ref{t:det-rounding}.

We remark that by Corollary~\ref{cor:2} of Section~\ref{sec:MSSC_random} and
Lemma~\ref{l:comparison} of Section~\ref{sec:GMSSC}, for any request $R$ with covering requirement $\mathrm{K}(R)=1$,
\[\sum_{i=1}^n\left(1-\sum_{j=1}^{i-1}\sum_{e \in R}A_{ej},0 \right)_+
\leq(1+\eps) \cdot\mathrm{FAC}_R(A)~~\text{for any }A\in\mathrm{DS}
\]
where $\eps$ is the parameter used in Definition~\ref{d:artificial_cost}. As a result, Theorem~\ref{t:det-rounding} follows directly by Theorem~\ref{t:rounding2}, which is stated below and proved in the next section. 
\begin{theorem}\label{t:rounding2}
Let $\pi_A \in [n!]$ denote the permutation of elements produced by Algorithm~\ref{alg:drand} when the doubly stochastic matrix $A \in \mathrm{DS}$ is given as input. Then for any request $R$ with $|R|\leq r$ and $\mathrm{K}(R)=1$, 
\[\acost(\pi_A,R) \leq 2(1+\alpha)^2r \cdot \sum_{i=1}^n \left(1-\sum_{j=1}^{i-1}\sum_{e \in R}A_{ej}\right)_+. \]
\end{theorem}

\subsection{Proof of Theorem~\ref{t:rounding2}}\label{s:4}
Consider a request $R \in [n^r]$ such that
\begin{equation}\label{eq:index}
(L-1)\cdot r + 1 \leq \acost(\pi_A,R) \leq L \cdot r
\end{equation}
for some integer $L$. Since $\mathrm{K}(R)=1$ this means that
the first element of $R$ appears between positions 
$(L-1)\cdot r + 1$ and $L\cdot r$ in permutation $\pi_A$.

To simplify notation we set $\mathrm{Cost}(A,R):= \sum_{i=1}^n \left(1-\sum_{j=1}^{i-1}\sum_{e \in R}A_{ej}  \right)_+$.
To prove  Theorem~\ref{t:rounding2} we show the following, which can be plugged in (\ref{eq:index}) and give the result:
\[\mathrm{Cost}(A,R) \geq \frac{L}{2(1+\alpha)^2}.\]

Let $R_\ell$ denote the request of size $r$ composed by the elements lying from position $(\ell-1)\cdot r + 1$
to $\ell\cdot r$ in the produced permutation $\pi_A$. Recall the minimization problem of Step~$3$.  $R_\ell$ is a (1$+\alpha$) approximately optimal solution for that problem and thus its corresponding cost is at most (1+$\alpha$) times the corresponding cost of any other same-cardinality subset  of the remaining elements. Since in  $\pi_A$ 
all the elements of $R$ lie on the right of  position $(L-1)\cdot r$, all elements of $R$ are present at the $L$-th iteration and thus, 
\[\mathrm{Cost}(A,R_{L}) \leq (1+\alpha) \cdot  \mathrm{Cost}(A,R)\]
Moreover, by the same reasoning,
\[\mathrm{Cost}(A,R_\ell) \leq (1+\alpha)\cdot\mathrm{Cost}(A,R_L),\text{ for all } \ell = 1,\ldots,L.\]
Thus it suffices to show that
$\mathrm{Cost}(A,R_L) \geq L/2(1+\alpha)$. The latter is established in Lemma~\ref{l:technical}, which concludes the section.

\begin{lemma}\label{l:technical}
Let $R_1, R_2, \ldots , R_L$ be disjoint requests of size $r$ such that for all $\ell = 1, \ldots, L$, $\mathrm{Cost}(A,R_\ell) \leq(1+\alpha) \cdot  \mathrm{Cost}(A,R_L)$. Then,
\[\mathrm{Cost}(A,R_L) \geq \frac{L}{2(1+\alpha)}\]
\end{lemma}
\begin{proof}
For each request $R_\ell$ we define the quantity $B_{\ell i}$ as follows:
\[ 
B_{\ell i}= \left\{
\begin{array}{ll}
      \sum_{e \in R_{\ell}}A_{ei}  & \text{ if   } \sum_{j=1}^i\sum_{e \in R_\ell}A_{e j} < 1\\
      1 - \sum_{j=1}^{i-1}\sum_{e\in R_\ell}A_{ej} & \text{  if  } \sum_{j=1}^{i}\sum_{e \in R_\ell}A_{e j} \geq 1  \text{ and } \sum_{j=1}^{i-1}\sum_{e \in R_\ell}A_{e j} < 1\\
    0 & \text{  otherwise}
\end{array} 
\right. 
\]
\begin{observation}\label{obs:taB}
The following $3$ equations hold,
\begin{enumerate}
    \item $\sum_{i=1}^n B_{\ell i} = 1$.
    
    \item $B_{\ell i} \leq \sum_{e \in R_\ell} A_{ei}$.
    
    \item $\mathrm{Cost}(A,R_\ell) = \sum_{i=1}^n  \left ( 1 - \sum_{j=1}^{i-1} \sum_{e \in R_\ell}A_{e j}\right )_+ = \sum_{i=1}^n \left( 1 - \sum_{j=1}^{i-1}B_{\ell j}\right)$ 
\end{enumerate}
\end{observation}
\noindent Since $(1+\alpha)\cdot \mathrm{Cost}(A,R_L) \geq \mathrm{Cost}(A,R_\ell)
$ for all $\ell = 1, \ldots, L$,
\begin{eqnarray*}
\mathrm{Cost}(A,R_L) &\geq& \frac{1}{1+\alpha} \cdot \frac{1}{L}\sum_{\ell = 1}^L \mathrm{Cost}(A,R_\ell) = \frac{1}{1+\alpha} \cdot \left[ \frac{1}{L}\sum_{\ell = 1}^L\sum_{i=1}^n\left( 1 - \sum_{j=1}^{i-1}B_{\ell j} \right) \right]\\
&=& \frac{1}{1+\alpha} \cdot \left[n - \frac{1}{L}\sum_{\ell = 1}^L \sum_{i=1}^n \sum_{j=1}^{i-1}B_{\ell j}\right]\\
&=& \frac{1}{1+\alpha} \cdot \left[n - \frac{1}{L}\sum_{i=1}^n\sum_{j=1}^{i-1}C_j\right] ~~~\text{(where }C_j = \sum_{\ell = 1}^L B_{\ell j})\\
&=& \frac{1}{1+\alpha} \cdot \left[n - \frac{1}{L}\sum_{i=1}^n (n - i) \cdot C_i \right]=
\frac{1}{1+\alpha} \cdot
\left[
n - \frac{n}{L} \sum_{i=1}^n C_i + \frac{1}{L}\sum_{i=1}^n i \cdot  C_i\right]\\
\end{eqnarray*}
\noindent Observe that $\sum_{i=1}^n C_i = \sum_{i=1}^n \sum_{\ell = 1}^L B_{\ell i } = 
\sum_{\ell = 1}^L \sum_{i=1}^n  B_{\ell i } = L,
$ where in the last equality we used  $\sum_{i=1}^nB_{\ell i} = 1$. Thus we get that
\[\mathrm{Cost}(A,R_L)  \geq
\frac{1}{1+\alpha}\left[\frac{1}{L}\sum_{i=1}^n i \cdot  C_i\right]\]
\noindent To this end, to conclude the result, one can prove that
$\sum_{i=1}^n i \cdot C_i \geq L^2/2$ using  that $\sum_{i=1}^nC_i = L$ and  $C_i \leq 1$. $C_i \leq 1$ follows by the \textit{disjoint property} of the requests $R_1,\ldots,R_L$.
More precisely,
\begin{eqnarray*}\label{eq:1}
C_i &=& \sum_{\ell=1}^L B_{\ell i}\leq \sum_{\ell = 1}^L\sum_{r \in R_{\ell}}A_{ri}\\
&\leq& 
\sum_{e=1}^n A_{ei}
=1
\end{eqnarray*}
\noindent where the first inequality follows from Observation \ref{obs:taB} and the  last inequality by $R_1,\ldots,R_L$  not sharing any element.
\end{proof}

\subsection{Implementing Step~$3$ of Algorithm~\ref{alg:drand} in Polynomial-Time}\label{s:DP}

In this section we present a polynomial time algorithm implementing Step~$3$ of Algorithm~\ref{alg:drand}. More precisely, we present a \textit{Fully Polynomial-Time Approximation Scheme (FTPAS)} for the combinatorial optimization problem defined below, in Problem~\ref{p:1}.

\begin{problem}\label{p:1}
Given an $n\times n$ doubly stochastic matrix $A$ and a set of elements $\text{Rem} \subseteq \{1,\ldots,n\}$. Select the $r$ elements of $\text{Rem}$ ($R^\ast \subseteq \text{Rem}$ with $R^\ast = r$) minimizing, \[\sum_{i=1}^n \left( 1 - \sum_{j=1}^{i-1}\sum_{e \in R^\ast}A_{ej}\right)_+.\]
\end{problem}

\noindent In fact we present a $(1+\alpha)$-approximation algorithm for a slightly more general problem, Problem~\ref{p:2}. 

\begin{problem}\label{p:2}
Given a set of $m$ vectors $B_1,\ldots, B_m$, of size $n$ such that,
\[ 0 = B_{e1} \leq B_{e2} \leq \ldots \leq B_{en} = 1, \text{   for each e = 1,\ldots,m}\]
Select the $r$ vectors ($R^\ast \subseteq [m]$ with $R^\ast = r$) minimizing
\[\sum_{i=1}^n \left( 1 - \sum_{e \in R^\ast}B_{ei}\right)_+\]
\end{problem}

\noindent 
Setting $B_{ei} = \sum_{j=1}^{i-1}A_{ej}$, one can get Problem~\ref{p:1} as a special case of
Problem~\ref{p:2}.

\begin{theorem}\label{l:problem}
There exists a $(1+\alpha)$-approximation algorithm for Problem~\ref{p:2} that runs in $\Theta(n^4r^3/\alpha^2)$ steps.
\end{theorem}

The $(1+\alpha)$-approximation algorithm of Problem~\ref{p:2} heavily relies on solving the Integer Linear Program defined in Problem~\ref{p:3}.

\begin{problem}\label{p:3}
Given a set of $m$ triples of integers $(w_e,c_e,d_e)$ such that
$c_e,d_e \geq 0$ for each $e \in \{1,m\}$ and two positive integers $C,D$,
\begin{equation*}
\begin{array}{ll@{}ll}
\text{minimize}  & \displaystyle\sum\limits_{e=1}^{m} w_{e}x_{e} &\\
\text{subject to}& \displaystyle\sum\limits_{e=1}^m c_e x_{e} \geq C\\
&\sum\limits_{e=1}^m d_e x_{e} \leq D \\
&\sum\limits_{e=1}^m x_{e} = r \\
&x_{e} \in \{0,1\} &e=1 ,..., m
\end{array}
\end{equation*}
\end{problem}

\begin{lemma}\label{l:dynamic}
Problem~\ref{p:3} can be solved in $\Theta(n \cdot C\cdot D \cdot r)$ steps via Dynamic Programming.
\end{lemma}
\begin{proof}
Let $\mathrm{DP}(n,r,C,D)$ denotes the value of the optimal solution. Then
\[\mathrm{DP}(n,r,C,D) = \min\left( 
\mathrm{DP}(n-1,r-1,C-x_n,D-d_n),
\mathrm{DP}(n-1,r,C,D)
\right)\]
\end{proof}

In the rest of the section, we present the $(1+\alpha)$-approximation algorithm for Problem~\ref{p:2} as stated in Lemma~\ref{l:problem} using the algorithmic primitive of Lemma~\ref{l:dynamic}.

We first assume the entries of the input vectors
are multiples of small constant $\alpha << 1$, $B_{ei} = k_{ei} \cdot \alpha$ for some integer 
$k_{ei}$. Under this assumption we can use the algorithm (stated in Lemma~\ref{l:dynamic}) for Problem~\ref{p:3} to find 
the \textit{exact} optimal
solution of Problem~\ref{p:2} in $\Theta(n^2 r / \alpha^2)$ steps.

More precisely, for a fixed index $k$, let $\mathrm{OPT}_k$ denotes the optimal solution among the set of vectors of size $r$ that additionally satisfy,
\begin{equation}\label{eq:costraints}
\sum_{e \in R}B_{e(k-1)}< 1 \text{ and }\sum_{e \in R}B_{ek}\geq 1    
\end{equation}
It is immediate that $\mathrm{OPT}= \argmin_{1 \leq k \leq n} \mathrm{OPT}_{k}$ and thus the problem of computing $\mathrm{OPT}$ reduces into computing $\mathrm{OPT}_k$ for each index $k$. We can efficiently compute $\mathrm{OPT}_k$ for each index $k$ by solving an appropriate instance of Problem~\ref{p:3}. To do so, observe that for any set of vectors $R$ satisfying the constraints of Equation~(\ref{eq:costraints}) for the index $k$,
\[ \sum_{i=1}^n \left( 1 - \sum_{e\in R} B_{ei}  \right)_+
= \sum_{i=1}^{k-1} \left( 
1 - \sum_{e \in R}B_{ei}
\right)
=\sum_{e \in R}
\underbrace{
\sum_{i=1}^{k-1} \left( 
\frac{1}{r}- B_{ei}
\right)}_{w_e}
\]
\noindent where the first equality comes from the fact that $B_{e1} \leq \ldots \leq B_{en}$. It is not hard to see that 
$\mathrm{OPT}_k$ can be computed via solving the instance of Problem~\ref{p:3} with triples $\left(w_e = \sum_{i=1}^{k-1} \left( 
\frac{1}{r}- B_{ei}
\right),c_e = B_{e(k-1)},
d_e = B_{ek}
\right)$ for each $e = 1 \ldots, m$, $D=1$ and $C=1$. Moreover by Lemma~\ref{l:dynamic} this is done in $\Theta(nr/\alpha^2)$ steps. Thus the overall
time complexity in order to compute the optimal solution of Problem~\ref{p:2} (in case the entries $B_{ei}$ are multiplies of $\alpha$)
is $\Theta(n^2r/\alpha^2)$.

We now remove the assumption that the entries $B_{ei}$ are multiples of $\alpha$ via relaxing the optimality guarantees by a factor of $(1+\alpha)$. We first construct a new set of vector with entries \emph{rounded} to the closest multiple of $\alpha$, $\hat{B}_{ei} = \left \lfloor{B_{ei}/\alpha}\right \rfloor \cdot \alpha$ and solve the problem as if the entries where $\hat{B}_{ei}$ in $\Theta(n^2r/\alpha^2)$ steps. The quality of the produced solution, call it $\mathrm{Sol}$ can be bounded as follows

\begin{eqnarray*}
\sum_{i=1}^n \left( 1- \sum_{e \in \mathrm{Sol}} B_{ei}  \right)_+
&\leq& \sum_{i=1}^n \left( 1- \sum_{e \in \mathrm{Sol}} \hat{B}_{ei} \right)_+\\
&\leq& \sum_{i=1}^n \left( 1- \sum_{e \in \mathrm{Sol}} \hat{B}_{ei}  \right)_+\\
&\leq& \sum_{i=1}^n \left( 1- \sum_{e \in \mathrm{OPT}} (B_{ei}-\alpha) \right)_+\\
&\leq& \sum_{i=1}^n \left( 1- \sum_{e \in \mathrm{OPT}} B_{ei}  \right)_+ + nr \cdot \alpha
\end{eqnarray*}
\noindent Setting $\alpha := \alpha'/nr$, we get that
\[\sum_{i=1}^n \left( 1- \sum_{e \in \mathrm{Sol}} B_{ei} \right)_+ \leq 
\sum_{i=1}^n \left( 1- \sum_{e \in \mathrm{OPT}} B_{ei} \right)_+ + \alpha' \leq
(1+\alpha')\sum_{i=1}^n \left( 1- \sum_{e \in \mathrm{OPT}} B_{ei}  \right)_+
\]
\noindent since $B_{e1} = 0$ for all $e$.
Thus, the overall time needed to produce a $(1+\alpha')$-approximate solution is $\Theta(n^4 r^3/(\alpha')^2)$, proving the result.

\subsection{A simple heuristic for Problem~\ref{p:1}}\label{s:heuristic}
In this section we present a simple heuristic for Problem~\ref{p:1} that can be a good alternative of the algorithm elaborated in Section~\ref{s:DP}. We remark that Algorithm~\ref{alg:heuristic} may provide highly sub-optimal solutions in the worst case however our experiments suggest that it works well enough in practice. As explained in Section~\ref{s:exp}, in our experimental evaluations we use this heuristic to implement Step~$3$ of Algorithm~\ref{alg:drand}. This was done since this heuristic is easier and faster to implement.

\begin{algorithm}[H]
  \caption{A simple heuristic for Problem~\ref{p:1}}\label{alg:heuristic}
  \textbf{Input:} A doubly stochastic matrix $A \in \mathrm{DS}$.\\
  \textbf{Output:} A set $R$ (of $r$ elements) approximating Problem~\ref{p:1}
  \begin{algorithmic}[1]

  \State $R = \varnothing $
  \State $\text{Target} = \underbrace{(1,\ldots,1)}_{n}$
  \For{ $\ell = 1$ to $r$ }

    \State $e_{\ell} \leftarrow \argmin_{e \in \{1,\ldots,n\}/R} \left( \sum_{j=1}^n \max( \text{Target}[j] - \sum_{s=1}^{j-1}A_{ej} , 0) \right)$
  
    \State $R \leftarrow R \cup \{e_\ell\}$
    
    \State $\text{Target} \leftarrow \left(\text{Target} - (A_{e1},\ldots,A_{en}) \right)_+$
  \EndFor
 
  \State Output the set of elements $R$.
  
    \end{algorithmic}
\end{algorithm}

\end{appendices}

\end{document}